\documentclass[11pt,a4paper]{article}

\usepackage[utf8]{inputenc}
\usepackage[english]{babel}
\usepackage{amsmath, amsfonts, amssymb, amsthm}
\usepackage{graphicx}
\usepackage{geometry}
\usepackage{hyperref}
\usepackage{booktabs}
\usepackage{cite}

\newtheorem{lemma}{Lemma}

\title{How Many Training Samples Are Needed for the Inverse Kinematics Solutions by Artificial Neural Networks}

\author{
    Dong-Won Lim\thanks{ORCID: 0000-0002-3086-1237} \\
    The University of Suwon, Hwaseong 18323, Republic of Korea \\
    \texttt{dwlim@suwon.ac.kr}
}

\date{}

\begin{document}

\maketitle

\begin{abstract}
Inverse Kinematics (IK) plays a critical role in robotic motion planning and control. The IK solutions of a robot manipulator could be done by conventional ways such as geometric, algebraic, or Jacobian methods, which have drawbacks. The Artificial Neural Networks (ANNs) have become a promising alternative for approximating IK solutions due to their generalization ability and computational efficiency. This approach basically trains only a few samples of the end effector that are recorded for the solution of the IK problem. However, a fundamental question remains: how many training samples are sufficient to achieve reliable and accurate IK predictions? This study investigates the mathematical framework of relating the size of training datasets and the accuracy of ANN-based IK solvers. Using an articulated robotic manipulator, we generate varying amounts of joint–position pairs to train feedforward neural networks and assess their accuracy, convergence, and generalization capability. The results reveal more training samples than 125 did not contribute to the improvement of the model efficiency that the comparable measure dealing with the approximation accuracy over the sampling size, offering valuable insight into data efficiency. This work provides practical guidance for optimizing the data sizing of ANN solutions, balancing computational cost and model accuracy for real-world robotic applications.
\end{abstract}


\section{Introduction}
In the field of function approximation using artificial neural networks (ANNs), the selection and quantity of training samples play an important role in determin-ing the accuracy and generalization capability of the learned model. As universal function approximators, ANNs are inherently data-driven. The ability of these networks to capture the underlying nonlinear relationships within a system is fundamentally limited by the quality, distribution and density of the training dataset. The training samples have been shown to influence two key aspects of ANNs during the optimization process. First, they impact the convergence behavior of the ANNs. Second, they affect the approximation error, which is the discrepancy between the true function and the model's prediction. Insufficient or poorly distributed training data can lead to high local errors in regions of the input space that are underrepresented. This can cause the model to extrapolate poorly or overfit to noisy patterns. In contrast, an excessively large dataset may result in unnecessary computational expenditure and may also exhibit redundancy or imbalance if it is not meticulously designed. It is imperative to comprehend the interrelation between the complexity of training samples and the occurrence of approximation error if efficient and accurate ANN models are to be developed.

Similar concerns about data quantity and efficiency arise in broader artificial intelligence research, particularly in the development of large language models (LLMs). Over the years, LLMs have been trained with exponentially increasing datasets, reaching 300 billion tokens in GPT-3 [1], 1.4 trillion tokens in Chinchilla [2], 13 trillion tokens in GPT-4 [3], and over 15 trillion tokens in Llama 3 [4]. However, researchers anticipate that this trend will plateau due to the finite availability of fresh textual data. As pointed out in [5], AI research is shifting toward optimizing smaller models with enriched synthetic datasets, such as Phi-3-Mini [6], to maintain performance while reducing unnecessary data accumulation. This trend directly parallels the challenges faced in ANN-based Inverse Kinematics (IK) modeling, where excessive training samples may lead to computational inefficiencies without substantial improvements in accuracy. 

The calculation of IK for robot manipulators is a fundamental problem in motion control. When using ANNs, the computation process can be streamlined into three primary steps: training, learning the inverse kinematics function, and obtaining joint angles for desired target locations of the end-effector. Unlike conventional geometric, algebraic, or Jacobian-based methods, which require explicit mathematical modeling, ANN-based approaches avoid cumbersome parameter calibration and issues related to singularities [7-13]. Traditional IK computations often involve Denavit-Hartenberg (D-H) parameters, which can be difficult to adjust for structural inconsistencies, such as mechanical deflections or joint looseness. Furthermore, conventional methods may lead to inefficient working paths or fail entirely in certain singularity postures. 

In ANN-based IK solutions, previous studies have demonstrated that robotic systems can learn their IK characteristics directly via ANNs without requiring explicit mathematical formulations [14]. Some approaches incorporate feedback from joint angular configurations into the ANN learning process [15], allowing robots to adaptively improve their performance. The number and quality of training samples directly affect the approximation error, which consists of bias (due to insufficient model complexity), variance (due to sensitivity to training data), and irreducible noise. In the context of supervised learning, the generalization performance of an ANN hinges not only on the network architecture and optimization algorithm but also critically on the representativeness and density of the training dataset [9]. Different techniques have been proposed to improve the efficiency of training data acquisition, such as active learning, uniform grid sampling in task space, and error-driven sample refinement [11], which aim to expose the ANN to a more diverse and representative set of examples.
 
While many works in the literature have adopted relatively large volumes of training samples, studies vary significantly in their sample sizes. Duka [12] used 1,000 randomly generated sets of joint angle values with 100 hidden neurons. However, determining the optimal number of training samples remains an open question. Some researchers have proposed integrating uncertainty estimation techniques—such as Bayesian neural networks or ensemble models—to identify low-confidence predictions and guide adaptive sampling [16]. Furthermore, excessive data may lead to computational inefficiencies or overfitting, while inadequate sampling may result in poor generalization and unreliable motion control. In [17], the author tried to answer for the bounding problem on the number of samples, but it was not on the general problem setting. Alwosheel et. al [18] argued that it is unknown what sample size requirements are appropriate when using ANN. Mehrotra et. al [19] addressed the relationship between the number of hidden layer nodes and the number of samples needed for effective learning, but not for the approximation such as IK. Thus, balancing the trade-off between sample size and model efficiency is critical in ANN-based IK modeling. 

This study aims to investigate the relationship between training sample size and model accuracy in ANN-based IK approximation. By examining statistical learning principles, approximation theory, and robotics-specific command-follow setting, this research seeks to develop a generalization methodology so that one can secure the ANN model efficiency. 

The manuscript contributes to the field in the following ways:
\begin{itemize}
  \item Mathematical Framework: This research derives explicit bounds on approximation error using Lipschitz continuity and ANN structure, which is not found in other IK studies. 
  \item Sample Efficiency Metric: The introduction of the error-to-sample-spacing ratio is novel and practical. 
  \item Architecture Independence: The lemmas in this study are not tied to specific ANN configurations. 
  \item Empirical Validation: This study validates the theory with simulations on a 3-DOF manipulator, showing saturation of error improvement beyond 125 samples.
\end{itemize}

This paper is composed of the following sections. In Section 2, mathematical formulation is presented for the problem given. Section 3 is devoted to the application of theories developed in Section 2 for the IK problem. In Section 4, simulation configurations with problem set-up were explained, and the results are explained and discussed in Section 5. Afterward, the concluding remarks are given in Section 6.

\section{Problem Formulation}
The overview of the signal flow for this problem formulation is shown in Fig. 1. Three end-effector variables in the Cartesian coordinate as inputs are normalized for ANN, and three joint angles are yielded as outputs. Output positions are feedback and compared with inputs, and this is done by the geometric forward kinematics (FK) [20], and errors $e$ between inputs and outputs by ANN and FK functions are obtained. The error $e$ is considered for the estimate accuracy.

\begin{figure}[htbp]
\centering
\includegraphics[width=0.8\textwidth]{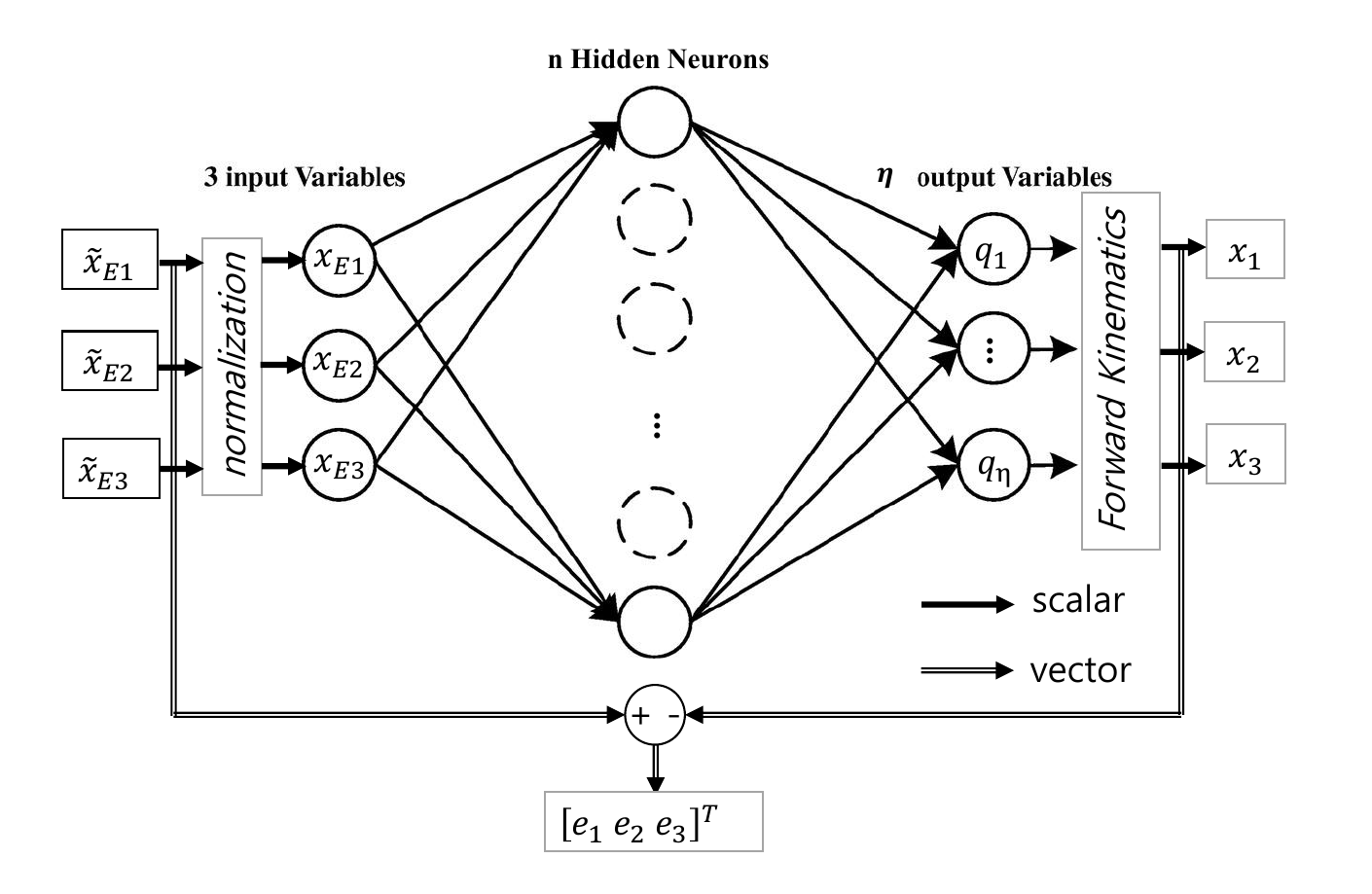}
\caption{Artificial Neural Networks for the Inverse Kinematics Problem with Feedback for Error Calculation}
\label{fig:ANN4IKchart}
\end{figure}

Mathematically, the relationship between training sample size and approximation error can be framed within the context of statistical learning theory. Consider the true inverse kinematics $f^*$ as
\begin{equation}
f^* : \mathbf{X} \subset \mathbb{R}^3 \to \mathbf{Q} \subset \mathbb{R}^\eta
\end{equation}
mapping from Cartesian space $\mathcal{X}$ to joint space $\mathcal{Q}$. 
And consider the following approximation model 
\begin{equation}
\hat{f}_n(\mathbf{X}) = \sum_{j=1}^{l}\sum_{i=1}^{3} w_{jk}^{(2)}\phi\left(w_{ij}^{(1)} x_i+b_j^{(1)}\right)+b_k^{(2)} \quad \text{for } i\in\{1,2,3\}, \ j\in\{1,2,\cdots,l\}, \text{ and } k\in\{1,2,\cdots,\eta\}
\end{equation}
, which is parameterized by the input $i$ and hidden node $j$ weights $w$ and hidden node $j$ and output node $k$ biases $b$. The model is learned by a neural network trained on $n$ samples. The function $\phi: \mathbb{R} \to \mathbb{R}$ is the activation function in the hidden layer.
The expected error $e$ typically decreases as $n$ increases, subject to the capacity of the network and the complexity of $f^*$ as
\begin{equation}
e = \mathbb{E}\left[ \| g(f^*(X)) - g(\hat{f}_n(X)) \|_2^2 \right]
\end{equation}
where $g(\cdot)$ is the forward kinematics function.

According to classical results in approximation theory, for a network with fixed architecture, it is reported that the convergence rate of the error often follows
\begin{equation}
O\left(\frac{1}{n^{-\alpha}}\right) \quad \text{for some } \alpha > 0
\end{equation}
, depending on the smoothness of the target function [21].

In this paper, the main focus is on the theoretical findings on the convergence scale, which depends on the number of training samples $n$. Therefore, it is essential to relate the expected error in Eq. (3), appropriately. To quantify the differences, consider a training data point $X_E$ and a test point incremented from $X_E$ by $\Delta X$ to examine the risk of the deviation by its approximation. Now, the approximation by the ANN model $\hat{f}_n(\cdot)$ can be expressed by the inequality which is in the form of the Lipschitz inequality as
\begin{equation}
\| g(\hat{f}_n(X_E + \Delta X)) - g(\hat{f}_n(X_E)) \|_2 \le \gamma \| \Delta X \|_2
\end{equation}
where $\Delta X \in \mathbb{R}^3$ in $[0,1]$ is the normalized Cartesian DOFs, and $\gamma$ is the Lipschitz constant. It is general to have a test point which is not trained in the network. The risk at the test point can be derived as of the following.

\begin{lemma}
(Upper bounding the approximation error of the ANN function at the test point $X_T = X_E + \Delta X$) For the upper bounding error, by the Lipschitz inequality, the error at $X_T$ is bounded as
\begin{equation}
e_{X_T} \le (\gamma^2 + 1) \| \Delta X \|_2^2.
\end{equation}
\end{lemma}

\begin{proof}
The core of this proof is the geometric relationship between trained and untrained data points. As mentioned earlier, no matter how small the number of training samples is, the trained points are very close to the exact IK solutions. Therefore,
\[
\| g(\hat{f}_n(X_E + \Delta X)) - g(\hat{f}_n(X_E)) \|_2 \approx \| g(\hat{f}_n(X_E + \Delta X)) - X_E \|_2
\]
The right-hand side becomes,
\[
\| g(\hat{f}_n(X_E + \Delta X)) - (X_E + \Delta X) + \Delta X \|_2 = \| X_T - g(\hat{f}_n(X_T)) - \Delta X \|_2
\]
By the triangle inequality with both sides squared,
\[
\left| \| X_T - g(\hat{f}_n(X_T)) \|_2^2 - \| \Delta X \|_2^2 \right| \le \| X_T - g(\hat{f}_n(X_T)) - \Delta X \|_2^2
\]
Therefore,
\[
\left| \| g(f^*(X_T)) - g(\hat{f}_n(X_T)) \|_2^2 - \| \Delta X \|_2^2 \right| \le \gamma^2 \| \Delta X \|_2^2.
\]

Using the error notation, the above inequality becomes
\[
\left| e_{X_T} - \| \Delta X \|_2^2 \right| \le \gamma^2 \| \Delta X \|_2^2.
\]
The relation can be either
\begin{equation*}
e_{X_T} - \|\Delta X\|_2^2 \le \gamma^2 \|\Delta X\|_2^2 \quad \text{or} \quad \|\Delta X\|_2^2 - e_{X_T} \le \gamma^2 \|\Delta X\|_2^2
\end{equation*}
However, we are only interested in the upper bound, and the error at the test point becomes
\begin{equation*}
e_{X_T} \le (\gamma^2 + 1) \|\Delta X\|_2^2.
\end{equation*}
This is the end of the proof. \hfill $\square$
\end{proof}

\begin{lemma}
(Lipschitz constant $\gamma$) Within the assumption that the activation of the hidden and output layer is the ReLU and linear functions, $\gamma$ can be chosen as
\begin{equation}
\gamma = \sqrt{3} \left\| \frac{\partial \hat{f}_n}{\partial X} \right\|_\infty
\end{equation}
\end{lemma}

\begin{proof}
For the output $\mathbb{Q}$ of $\hat{f}_n$, in fact only one activation function appears in the calculation process from the input as
\begin{equation}
\phi(x) := [\max\{0, x_1\}, \max\{0, x_2\}, \dots, \max\{0, x_l\}]^T
\end{equation}
Here, $u_j$, $w_{ij}^{(l)}$, and $b_j^{(l)}$ be the $j$-th hidden neuron, the weight from $i$-th to $j$-th elements at $l$-th layer, and the $j$-th bias at $l$-th layer, respectively, and consider a shallow network as described earlier,
\begin{align}
u_j &= \phi\left(\sum_{i=1}^{3} w_{ij}^{(1)} x_i + b_j^{(1)}\right), \quad j \in \{1, 2, \dots, l\}, \\
q_k &= \sum_{j=1}^{\eta} w_{jk}^{(2)} u_j + b_k^{(2)}, \quad k \in \{1, 2, \dots, \eta\}
\end{align}
where $X = [x_1, x_2, x_3]^T$.

The approximation function $\hat{f}_n$ is not continuously differentiable on $\mathbb{R}^3$, but this condition holds only at $x = 0$. And the function is locally Lipschitz on $\mathbb{R}^3$. The partial derivative matrix of $\hat{f}_n$ per Cartesian DOF is the inverse Jacobian, which is given by
\begin{equation}
\left[\frac{\partial \hat{f}_n}{\partial X}\right] = \begin{bmatrix}
\frac{\partial q_1}{\partial x_1} & \frac{\partial q_1}{\partial x_2} & \frac{\partial q_1}{\partial x_3} \\[6pt]
\frac{\partial q_2}{\partial x_1} & \frac{\partial q_2}{\partial x_2} & \frac{\partial q_2}{\partial x_3} \\[6pt]
\vdots & \vdots & \vdots \\[6pt]
\frac{\partial q_\eta}{\partial x_1} & \frac{\partial q_\eta}{\partial x_2} & \frac{\partial q_\eta}{\partial x_3}
\end{bmatrix}
\end{equation}

Using the infinity induced matrix norm in $\mathbb{R}^3$ for $\left[\frac{\partial \hat{f}_n}{\partial X}\right]$ from Eq. (4), we have
\begin{equation}
\left\| \frac{\partial \hat{f}_n}{\partial X} \right\|_\infty = \max_{k \in \{1,2,\dots,\eta\}} \left\{ \left| \frac{\partial q_k}{\partial x_1} \right| + \left| \frac{\partial q_k}{\partial x_2} \right| + \left| \frac{\partial q_k}{\partial x_3} \right| \right\}
\end{equation}

Because the output $q$ has the same format, but only the weights and biases are varying, the output partial derivative can be estimated in a general form from Eqs. (2 and 3). Using the chain rule, knowing $\phi^{\prime}(x)=0$ (for $x<0$) or $1$ (for $x>0$), the partial derivative of an output with respect to an input in a general form, becomes
\begin{equation}
\frac{\partial q_{k}}{\partial x_{i}} = \frac{\partial q_{k}}{\partial u_{j}}\frac{\partial u_{j}}{\partial x_{i}} = w_{jk}^{(2)} \cdot (1 \text{ or } 0).
\end{equation}
Therefore, the maximum magnitude of the partial derivative of $g$s with respect to $x$s can have its upper bound by
\begin{equation}
\left|\frac{\partial q_{k}}{\partial x_{i}}\right| \le w_{jk}^{(2)}, \quad j \in \{1,2,\dots,l\} \text{ and } k \in \{1,2,\dots,\eta\}.
\end{equation}
And using the norm inequality relation and the chain rule, 
\begin{equation}
\left\|\frac{\partial g(\hat{f}_{n})}{\partial X}\right\|_{\infty} = \left\|\frac{\partial g}{\partial Q} \cdot \frac{\partial \hat{f}_{n}}{\partial X}\right\|_{\infty} \le \left\|\frac{\partial g}{\partial Q}\right\|_{\infty} \cdot \left\|\frac{\partial \hat{f}_{n}}{\partial X}\right\|_{\infty} \le \left\|\frac{\partial \hat{f}_{n}}{\partial X}\right\|_{\infty}
\end{equation}
for $\left\|\frac{\partial g}{\partial Q}\right\|_{\infty}$ is atmost 1 due to the normalization, which values will convert back to the Cartesian space.

And the Lipschitz relation employed here can be re-written as
\begin{equation} 
\left\|\frac{g(\hat{f}_{n}(X_E + \Delta X)-\hat{f}_{n}(X_E))}{\Delta X}\right\|_{2} \le \sqrt{3}\left\|\frac{\partial g(\hat{f}_{n})}{\partial X}\right\|_{\infty} \le \sqrt{3}\left\|\frac{\partial \hat{f}_{n}}{\partial X}\right\|_{\infty} = \gamma
\end{equation}
due to the dimension of $X$ is three. This is the end of the proof. \hfill $\square$
\end{proof}

Based on Lemmas 1 and 2, the upper bounding error, by the Lipshcitz inequality, the error at $X_T$ can be expressed as of the following.
\begin{equation}
e_{X_T} \le \left(3\left\|\frac{\partial \hat{f}_{n}}{\partial X}\right\|_{\infty}^2 + 1\right) \|\Delta X\|_2^2.
\end{equation}

\section{Technical Results}
To investigate the theoretical bounding on the ANN approximation, a 3 degree-of-freedom (DOF) articulated robot manipulator (which three links are 70 mm long) shown in Fig. 2 is taken as an example for this study. The figure shows the manipulator tries to follow the reference path (which is the heart shape), but it is not able to follow it. The manipulator moves the end-effector with three revolute actuators (RRR), which the angles in $\mathcal{Q}$ are in the range of $q_{1} \in [\pi/3, 2\pi/3]$, $q_{2} \in [-2\pi/3, -\pi/3]$, $q_{3} \in [-2\pi/3, -\pi/3]$. The base (or first) joint rotates a yaw angle, which changes the aiming of the end-effector. The other two joints rotate pitch angles governing the height and radial length of the tip. The desired target location of the end-effector $X = [x_1, x_2, x_3]$ can be determined by the analytic forward kinematics (FK) [20] for given joint angle values $\mathcal{Q}$. The working range in the Cartesian (global) coordinates covers $x_1 \in [20, 100]$ mm, $x_2 \in [20, 100]$ mm, $x_3 \in [-10, 60]$ mm.

\begin{figure}[htbp]
\centering
\includegraphics[width=0.6\textwidth]{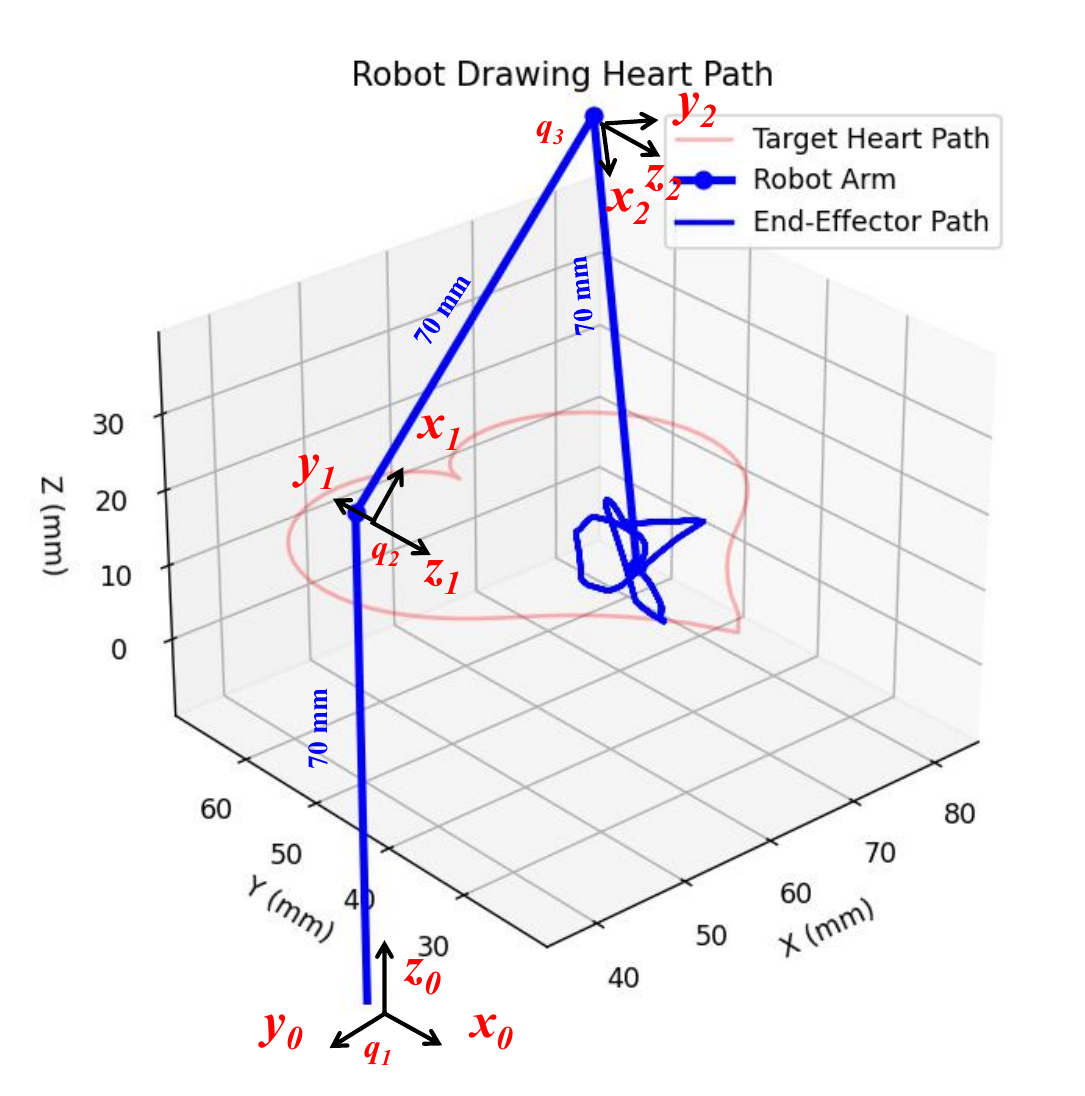}
\caption{The robot configuration with the local coordinate systems}
\label{fig:robot_config}
\end{figure}

The derivation of the bounding error requires the knowledge of weights in the network as shown before. The weights, $w_{ij}$s and $w_{jk}$s, are changing recursively over the training process, and converge to certain values typically in the range $[-5, 5]$ as the networks get trained. Therefore, it is reasonable that the average weight $\bar{w}$ is taken into consideration in places of $w_{ij}$s and $w_{jk}$s, and $\bar{w}$ is dealt with empirically in this case study. In this regard, a conservative bound to some extent will be obtained. Then, the derivative inequality relation becomes
\begin{equation}
\left| \frac{\partial q_k}{\partial x_i} \right| \le \bar{w}.
\end{equation}
Hence, by Eqs. (16), (18), and (22),
\begin{equation}
\left\| \frac{\partial f}{\partial X} \right\|_\infty \le 3\bar{w}
\end{equation}

When $X$ is a trained and $X + \Delta X$ is an arbitrary untrained point (or the location which is not selected for the ANN training process), at the point $X$ the difference between $\hat{f}_n(X)$ and joint angles calculated by IK will be minimal if only the training is done satisfactorily. But $\hat{f}_n(X + \Delta X)$ can be varied from the IK exact solutions depending on the training resolution. In general, no matter how small the number of training samples is, the trained points are very close to the exact IK solutions. And the difference can be the maximum at the middle point between training points. By this observation, the maximum likelihood error occurs with $\Delta X$ of

\begin{equation}
\left|\Delta X\right|_2 = \frac{1}{2} \cdot \frac{1 - 0}{\sqrt[3]{n} - 1}
\end{equation}
where $n$ is the number of training samples. From Eq. (21) by Eqs. (23 and 24), the upper bound of the IK function for the articulated (RRR) robot by the ANN fitting can be estimated by
\begin{equation}
e_{X_T} \le \frac{27\bar{w}^2 + 1}{4(\sqrt[3]{n} - 1)^2}
\end{equation}
where $n$ is the main parameter to investigate the fitting error changes. The weight average is found from trained functions empirically, but it is taken as a constant or a parametric analysis variable in the estimation.

This error is to be taken into the accuracy evaluation in terms of $n$. The relationship in Eq. (25) is further discussed in the following simulation section by the case study of the 3-DOF manipulator.

\section{Simulation Study}
The number of samples at each translational DOF of the end-effector is varied. And only 10 hidden neurons are fixed as an ANN layout, and 5\% was used for both validation and test in training due to the small number of samples.

For samples, evenly divided end-effector positions in the Cartesian $X = [x_1, x_2, x_3]^T$ space are picked up. The joint angles, $Q = [q_1, q_2, q_3]^T$ vector, can be measured by setting the poses at given positions. However, for the convenience of the study, the geometric inverse-kinematics is used instead to calculate the $Q$ vectors for simulations by FK equations in [22]. Only the number of sampling is varied as explained in the following sections.

The calculation of IK by ANN is a relatively straightforward process. Firstly, the end-effector locations for various postures must be recorded (or the forward kinematics must be calculated from various angular positions). Secondly, the networks between the inputs of end-effector locations and outputs of joint angles as a function must be built. In the context of this simulation study, the networks under investigation employ a set of 16 hidden neurons as a configuration of ANNs. In view of the limited number of samples, 5\% of the data was allocated for the validation and testing stages during the training phase. As previously documented, the quantity of hidden neurons exerted limited influence on the outcomes [22]. As neural networks have the capacity to compute any function, it is possible to approximate the mapping between joint angle and Cartesian end-effector domains by means of neural networks without calculating the inverse kinematics [23].

The IK ANN model was implemented in Python 3.10, leveraging the TensorFlow 2.11.0 library, which provides high-level APIs for building and training deep learning models. The model construction utilized \texttt{tensorflow.keras} for defining layers, loss functions, optimizers, and training routines. All experiments were conducted on a laptop computer equipped with an Intel 11th Gen Intel(R) Core(TM) i5-1135G7 @ 2.40GHz CPU, 16GB RAM to ensure efficient training and visualization. The code was modular and well-documented to facilitate reproducibility and scalability to higher-dimensional problems or alternative robotic configurations. 

In this study, the ANN is designed to approximate the inverse kinematics (IK) mapping of a 3 degree-of-freedom (DOF) robotic manipulator. The network learns the transformation from the Cartesian space (desired end-effector posi-tions) to the joint space (corresponding joint angles) using supervised learning. The model architecture, hyperparameters, and training strategy are selected to balance model capacity, computational efficiency, and generalization performance. The model is trained using the mean squared error (MSE) loss function, a standard choice for regression tasks. The MSE penalizes larger deviations more heavily, which is useful in precision-demanding applications like robotics. The optimizer employed is the Adam algorithm, a first-order gradient-based method that combines the advantages of momentum and adaptive learning rates. It is widely used due to its fast convergence and robustness to hyperparameter settings. The model is trained for a maximum number of epochs, set to 500 for the sufficient saturation, but training can terminate earlier based on convergence. To avoid overfitting and reduce unnecessary computation, early stopping is selectively implemented with a patience of 10 epochs and a minimum delta improvement threshold. The network is trained in mini-batches of 8 samples to ensure efficient use of memory and stable gradient updates. The learning rate was set to be 0.001. Only the input variables are normalized to the range $[0, 1]$ using a MinMaxScaler to ensure faster and more stable training, for the linear activation is used for the output layer. For post-training, the model is evaluated on synthetic test trajectories, including predefined geometric paths such as heart or rectangular shapes in 3D space. The predictions are visualized alongside true trajectories, and the Euclidean error is computed to assess the model's accuracy.

The training set was determined to cover only the desired working domain. The sampling number, which is the same for each input, is a variable to examine the ANN performance. For example, when the sampling number is 125, the total number of training samples for three joint variables (one epoch) is $5\times5\times5 = 125$. In Figs. 4 and 5, the sampling points, being evenly spaced, are selected and shown by black dots. Again, these dots are the training samples of the neural network to be considered.

\begin{figure}[htbp]
\centering
\includegraphics[width=0.8\textwidth]{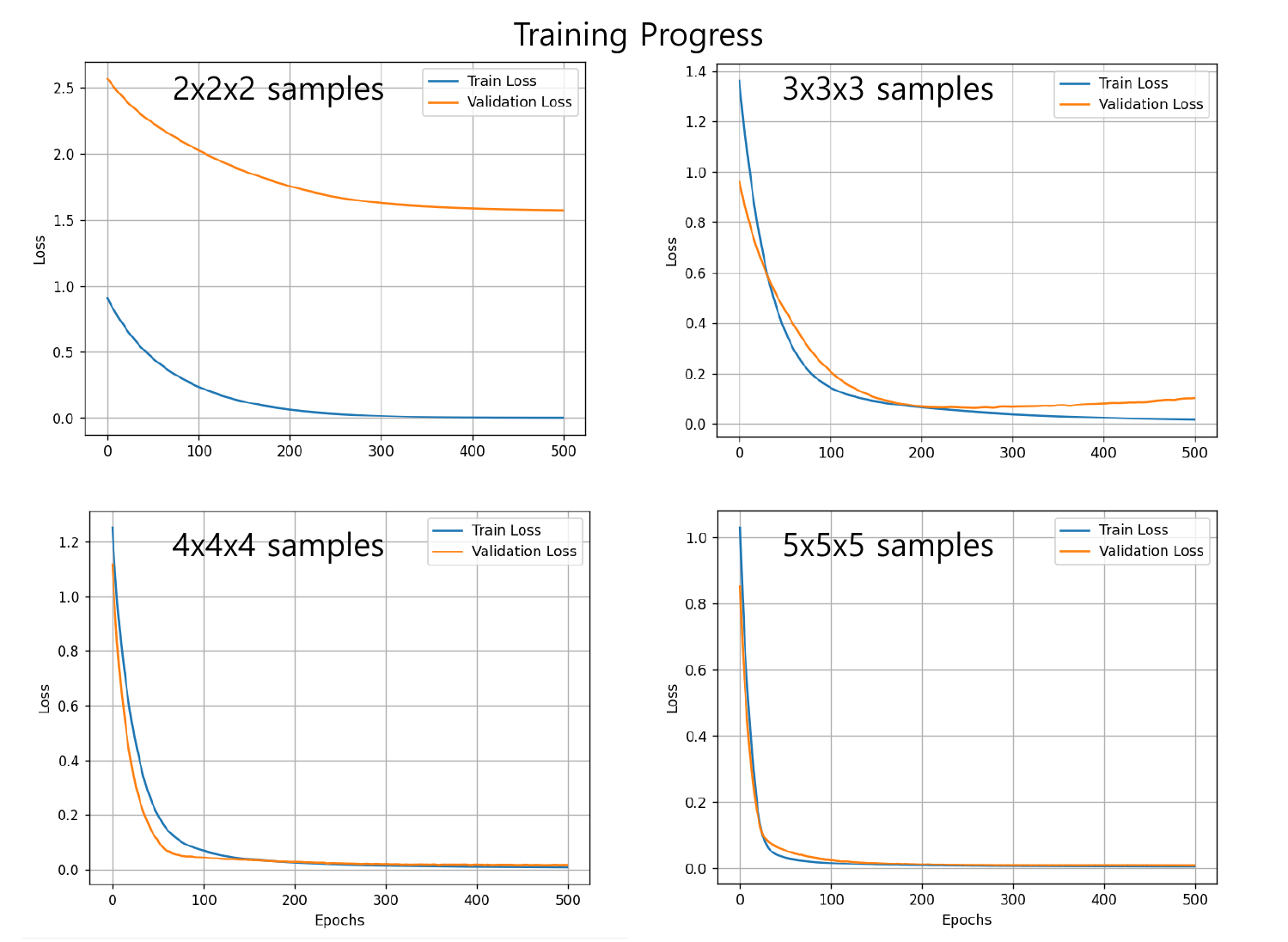}
\caption{Training Progress over the Epochs for Various Numbers of Samples}
\label{fig:training_progress}
\end{figure}

\section{Results and Discussion}

As demonstrated in Fig. 3, the training progress is displayed over the epoch. It is evident that as the quantity of training samples increases, the rate of reduction in loss diminishes, exhibiting a more rapid and substantial decline. As the number of samples increases from 64, the training converges to a negligible level after approximately 100 epochs. In contrast, the validation loss for 8 samples is approximately 1.5 at 500 epochs, suggesting that the performance of the function is inferior.

As illustrated in Figs. 4 and 5, the command is executed subsequent to the performance of the ANN fitting function, whereby the double rectangles are drawn in high and low heights. The red line indicates the targeted true path, while the blue dashed line illustrates the end-effector's following trajectory, as calculated by the ANN function. It is evident that black dots represent the training points. The training points are illustrated in the figures for the purpose of facilitating comparison. Due to the limited space available in this paper, the results obtained from the samples of 8, 27, 125 and 343 are displayed in the $xy$ and $yz$ planes of Figs. 4 (for 8 and 27 samples) and 5 (for 125 and 343 samples).

As demonstrated in the figures, a clear trend emerges: the efficacy of command-following performance is demonstrably enhanced by an increase in the number of samples. Indeed, the utilization of solely two extreme values per DOF in the training range is inadequate for ensuring optimal outcomes, as this results in trajectories that are scarcely depicted as rectangles. The superiority of three training points over two is evident in the enhanced command following, as illustrated by the similarity in shape to the objective character (see bottom of Fig. 4). When five samples are taken at each DOF, the outcomes in Fig. 5 are significantly enhanced, and the results demonstrate an acceptable trajectory shape.

As previously discussed, the joint angles for a desired location can be calculated by means of well-known analytic IK methods. However, it should be noted that the arctangent and arcsine functions in kinematic equations can give multiple values (i.e., the ambiguity). Furthermore, the inverse trigonometric functions give rise to the singularity problem. Conversely, as illustrated in the sub-figures of Fig. 5, the ANN IK solution consistently yields a unique, albeit approximate, result. The validity of this approximation is contingent on a number of parameters, including the number of training samples, the network layout structure, iteration algorithms, and other hyper-parameters.

\begin{figure}[htbp]
\centering
\includegraphics[width=0.8\textwidth]{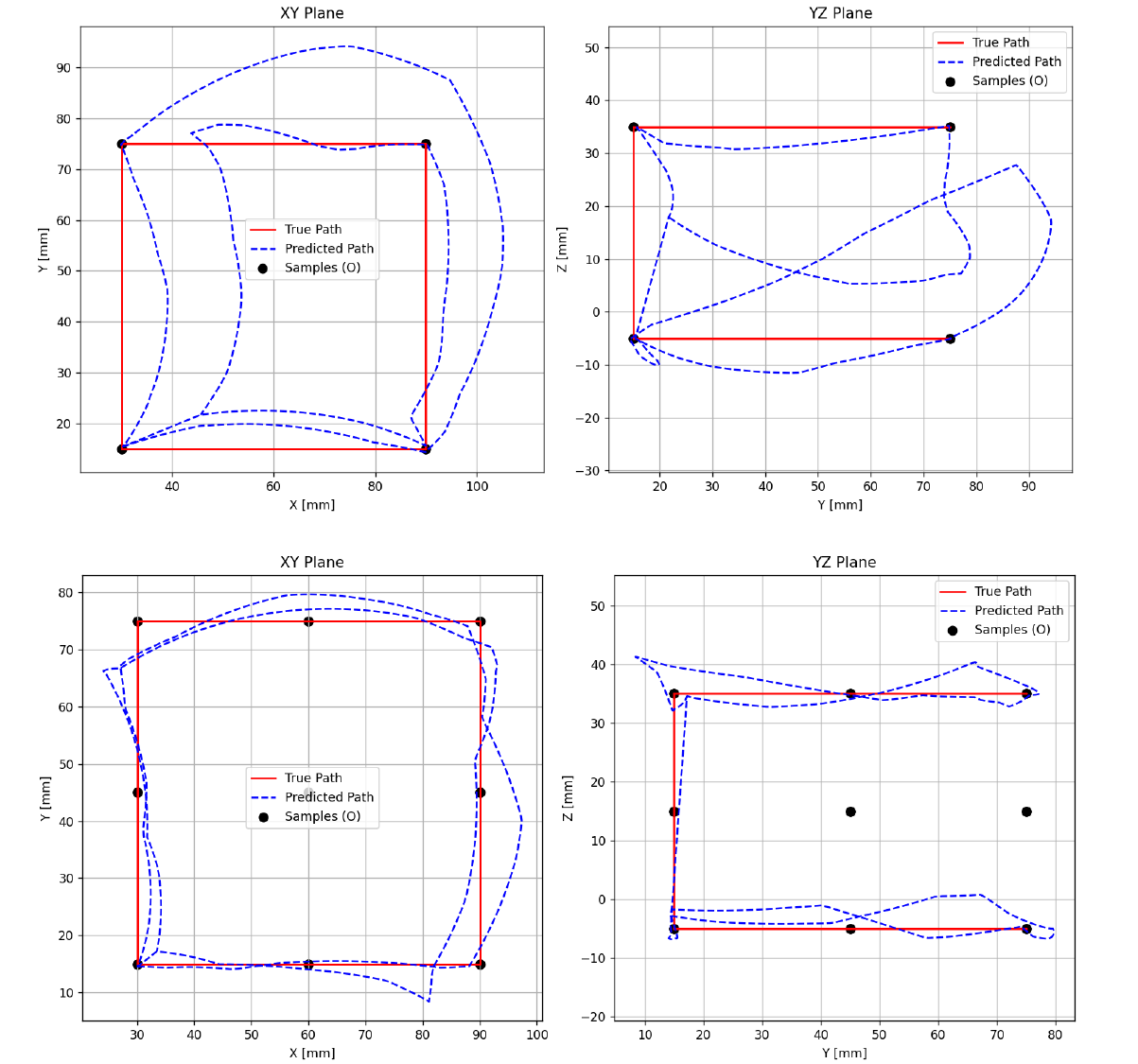}
\caption{Inverse Kinematics Artificial Neural Networks Function Performances}
\label{fig:IKANN1}
\end{figure}

\begin{figure}[htbp]
\centering
\includegraphics[width=0.8\textwidth]{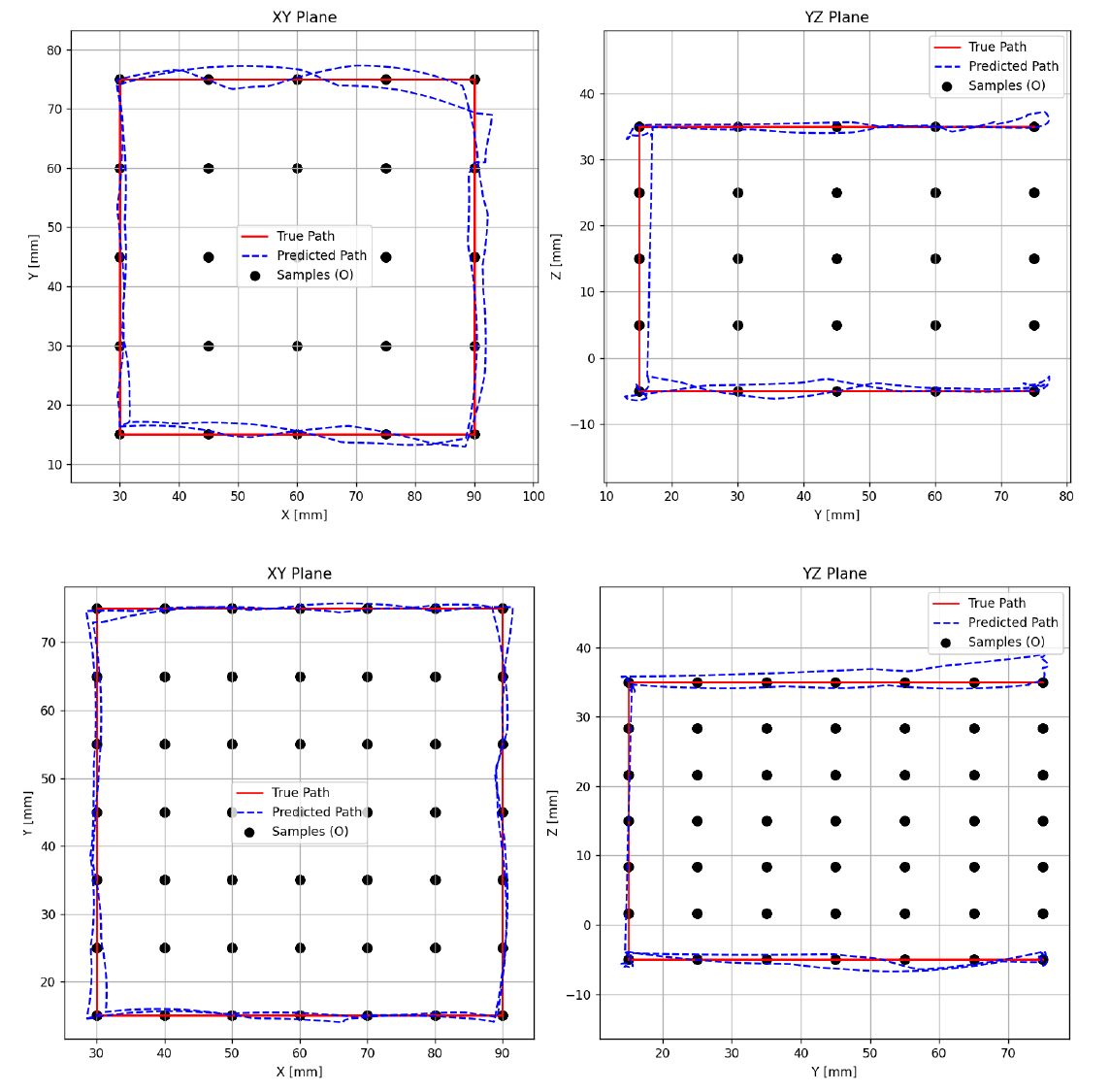}
\caption{Inverse Kinematics Artificial Neural Networks Function Performances}
\label{fig:IKANN2}
\end{figure}

As illustrated in Fig. 1, a series of simulation experiments was conducted, with the resulting errors summarized in Table 1. As demonstrated in earlier results, the errors decrease with an increase in the number of samples. The error estimates provided by Eq. (25) and the recorded weights from the experiments are displayed in the third column. It is important to note that the derivation in Eq. (25) was conducted with the normalization of parameters, and consequently, the estimates were rescaled into the experimental domain. As demonstrated in the table, the estimates show a high degree of similarity with the experimental results for up to the model on 64 samples. As described in Eq. (25), the estimates continue to decrease ($n$ is located in the denominator), however, the experimental errors appear to have not improved much.

\begin{table}[htbp]
\centering
\caption{Euclidean error by ANN in $X$ domain}
\label{tab:euclidean_error}
\begin{tabular}{cccc}
\toprule
\# of samples & $e_{\text{Exp.}}$ [mm] & $e_{\text{Est.}}$ [mm] & $e_{\text{Exp.}} / d$ \\ 
\midrule
$2 \times 2 \times 2 = 8$   & $19.31 \pm 3.75$ & 24.15 & 0.32 \\
$3 \times 3 \times 3 = 27$  & $5.31 \pm 0.70$  & 5.91  & 0.18 \\
$4 \times 4 \times 4 = 64$  & $2.73 \pm 0.12$  & 2.62  & 0.14 \\
$5 \times 5 \times 5 = 125$ & $2.17 \pm 0.20$  & 1.48  & 0.14 \\
$6 \times 6 \times 6 = 216$ & $1.85 \pm 0.24$  & 0.94  & 0.15 \\
$7 \times 7 \times 7 = 343$ & $1.52 \pm 0.21$  & 0.66  & 0.15 \\
$8 \times 8 \times 8 = 512$ & $1.23 \pm 0.12$  & 0.48  & 0.14 \\ 
\bottomrule
\end{tabular}
\end{table}

As illustrated in Fig. 5, the sparsity of the samples (black dots) is notable when compared with the blue dashed curves. As shown, the sampling distance of one DOF in the $X$ space is 15 mm, while the mean error is approximately 2.2 mm. A new metric has been proposed to facilitate a comparison between the accuracy and sampling size. This metric involves the division of error $e$ by spacing distance $d$ between samples, thereby offering a more nuanced assessment of the efficacy of the ANN model. The ratio is 0.14 for the model on 125 samples. The table indicates that the efficiency does not exhibit significant enhancement after selecting 64 samples for training, suggesting that the number 64 may be optimally set for the training size. Therefore, it can be concluded that the optimal error-to-sampling ratio for this problem set is 2.73 mm. In scenarios where the number of samples is limited, a greater margin of error is anticipated, resulting in a high error-sampling ratio. Conversely, an increase in the number of samples beyond the optimal number will result in a decrease in error, albeit to a marginal extent. This is due to the fact that the error will reach a state of saturation at a relatively narrow sampling distance, thereby resulting in a negligible decrease in the effective error. It is evident that utilizing a greater number of samples than 64 results in enhanced outcomes; however, the enhancement in precision becomes marginal.

\section{Conclusions}

The present study investigates the application of ANNs in the context of IK analysis for robot manipulator operation. The employment of ANNs rendered the utilization of IK calculations with no geometric information needed, but with only a limited number of training sample points being selected. ANN is responsible for the IK function, which is derived from the recorded locations of the end effector and the corresponding joint angles. However, an excess of sample points is impractical, and consequently, the study explores the mathematical relation of the training samples to the model accuracy. A framework was provided for the purpose of examining the extent to which the approximation function of the ANN can vary for different input variables, in relation to the bounds of the sample number. The mathematical derivations can be for the practical implementation of the ANN methodology.

In order to ascertain the validity of the proposed methodology, a 3 DOF drawing manipulator was set, and its IKs were solved by ANN. However, it was imperative to ascertain a sufficient number of training samples for the successful approximation, whilst identifying a minimum required number. The investigation focused on the relationship between the number of samples and the tracking performance, with a subsequent comparison of the tracking errors generated by the ANN with the results obtained through estimation methods. Consequently, the selection of four samples per DOF (20 mm long sampling resolution, i.e. the distance between samples) yielded a satisfactory approximation, exhibiting an accuracy of approximately 2.73 mm. The theoretical error bound provides a reliable estimate of the error when the number of samples is 125 or less. Given the error/spacing ratio of 0.14, the optimal number of samples that can be accommodated is found to be 64, in conjunction with an allowable distance error of 2.73 mm. It can be concluded that, in order to satisfy the requisite level of accuracy in the execution, the larger number of training samples should be determined with some sacrifice in the burden of bigger data size.

The present paper furnished the mathematical framework necessary to obtain the ANN displacement bound by means of a simulation case study. The estimate can be used to select the number of training samples. Whilst the present study concentrates on a 3-DOF manipulator and a fixed ANN architecture for the purposes of clarity and tractability, it should be noted that the proposed mathematical framework is by no means limited to this configuration. The error bounding lemmas are architecture-independent and can be extended to deeper networks or alternative activation functions. Furthermore, the sampling efficiency metric proposed in this study is applicable to higher-DOF manipulators, where the dimensionality of the input space increases the importance of training data density. Subsequent research will investigate the applicability of this framework across various robotic platforms and learning architectures, with the objective of validating its broader relevance.

The mathematical approach is also required to calculate the sampling bound as a whole, incorporating the recursive nature of ANN training. While the present study provides a static bound on approximation error, the training process involves iterative weight updates aimed at minimizing a loss function. By modeling this recursive optimization process, it may be possible to derive a dynamic sampling bound that adapts to the convergence behavior of the network. With this dynamic bound, for models with 125 and more samples, better estimates can be obtained. In addition, the weights that correspond to the case study must be analyzed, with numerical values assigned to the given conditions.

\end{document}